\RequirePackage{fix-cm}
\documentclass[notitlepage]{article}

\usepackage{arxiv}

\usepackage{amsmath,amsfonts,bm}

\def\eqref#1{equation~\ref{#1}}

\def\1{\bm{1}}

\DeclareMathAlphabet{\mathsfit}{\encodingdefault}{\sfdefault}{m}{sl}
\SetMathAlphabet{\mathsfit}{bold}{\encodingdefault}{\sfdefault}{bx}{n}

\DeclareMathOperator*{\argmin}{arg\,min}

\usepackage{amsmath,amssymb,amsthm,mathtools}
\usepackage[ruled,vlined,linesnumbered]{algorithm2e}
\SetAlgoNlRelativeSize{0}

\usepackage{booktabs}
\usepackage{placeins}
\usepackage{graphicx}
\usepackage{needspace}
\usepackage{xcolor}
\usepackage{tikz}
\usetikzlibrary{arrows.meta,calc}
\usepackage{microtype}
\usepackage[hidelinks]{hyperref}
\usepackage{bookmark}
\usepackage{url}
\usepackage{subcaption}
\usepackage{tabularx}

\newtheorem{theorem}{Theorem}
\newtheorem{lemma}{Lemma}

\definecolor{figInk}{HTML}{2B2B2B}
\definecolor{figReturn}{HTML}{332288}
\definecolor{figTemporal}{HTML}{CC6677}
\definecolor{figProjected}{HTML}{44AA99}
\definecolor{figGuide}{HTML}{D8DDE3}
\definecolor{figAxis}{HTML}{66666A}

\newcommand{\method}{\textnormal{\textsc{Damper}}}
\newcommand{\methodplain}{DAMPER}

\newcommand{\returnloss}{L_{\mathrm{R}}}
\newcommand{\temporalloss}{L_{\mathrm{T}}}

\title{\methodplain: Return-Prioritized Gradient\\
Control for Smooth Policies}

\author{
    Seokmin Ko \quad Taewon Goo \quad Kihyuk Hong \\
    \normalfont KAIST \\
    \normalfont\texttt{\{komin0407, tetrise9, kihyukh\}@kaist.ac.kr}
}

\date{}

\usepackage[
    backend=biber,
    style=numeric-comp,
    maxbibnames=9,
    maxcitenames=1,
    uniquelist=false,
    uniquename=false,
    isbn=false,
    sorting=nty,
    doi=false
]{biblatex}
\hypersetup{
    pdftitle={DAMPER: Return-Prioritized Gradient Control for Smooth Policies},
    pdfauthor={Seokmin Ko, Taewon Goo, Kihyuk Hong}
}

\begin{document}

\maketitle

\begin{abstract}
Actor--critic methods achieve strong performance in continuous control,
but their policies can produce highly oscillatory actions.
A common remedy is to add auxiliary smoothness losses.
However, their contribution can be negligible when their gradients are small
relative to the native actor gradient. Moreover, existing methods often
combine multiple auxiliary losses, complicating loss balancing without
necessarily improving the return--smoothness trade-off.
We introduce \method{} (Direction-Aware Magnitude-Controlled
Projection with Explicit Return Priority), which combines the native
actor gradient with a temporal-consistency gradient through
conflict-conditioned projection and adaptive magnitude control.
It removes the auxiliary component opposing the actor gradient
and scales the retained temporal direction relative to the actor
gradient norm, preserving positive alignment with the native
actor gradient.
Experiments with TD3 and SAC on six continuous-control tasks
show reduced action oscillation relative to the native agents
in all 12 task--backbone pairs and the best oscillation score
among the compared methods in eight, with task-dependent
return trade-offs. Our code is available at
\url{https://github.com/komin0407/DAMPER}.
\end{abstract}

\section{Introduction}
\label{sec:introduction}

Actor--critic algorithms are effective tools for continuous control, but high
episodic return does not imply a well-behaved control signal.  A learned policy
may vary its action sharply between adjacent time steps even when the system
state changes only slightly.  Simulation often tolerates these oscillations;
physical platforms do not.  Rapid control changes can waste energy, amplify
unmodeled dynamics, and increase actuator wear~\parencite{delapresilla2023oscillating}.  Learning policies that are
both effective and temporally smooth is therefore an important requirement for
deploying reinforcement learning on real systems.

\suppressfloats[t]
\begin{figure}[t]
    \centering
    \includegraphics[width=\linewidth]{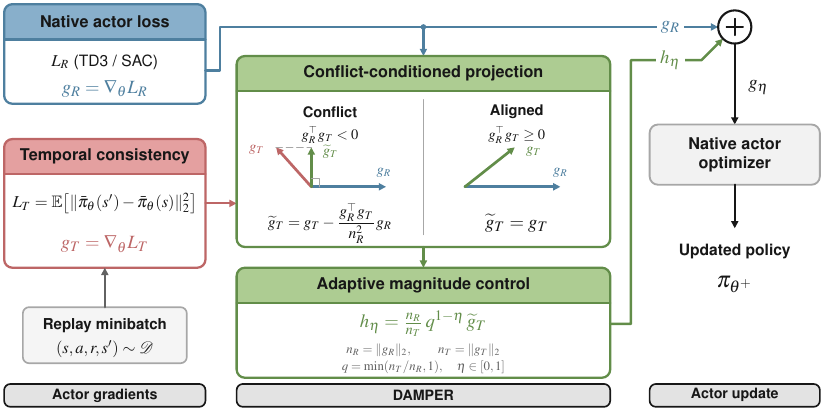}
    \caption{\textbf{Overview of \method{}.}
}
    \label{fig:damper-overview}
    \label{fig:method-geometry}
\end{figure}

Most existing approaches introduce smoothness through an auxiliary loss.
They regularize policy outputs at consecutive or nearby states, constrain local
policy sensitivity, or stabilize the critic gradients used by the actor~\parencite{mysore2021caps,kobayashi2022l2c2,lee2024gradcaps, kwak2026asap,lee2026pave}.
These methods can reduce action variation, but the smoothing they achieve is often limited. Moreover, many existing methods introduce two or more auxiliary losses to control smoothness \parencite{mysore2021caps,kobayashi2022l2c2,kwak2026asap,lee2026pave}, which complicates loss balancing without necessarily improving the return–smoothness trade-off, as illustrated by our comparison with a tuned single-loss baseline in Section~\ref{sec:scalarization_sweep}. 

We approach this problem as primary--auxiliary gradient optimization.  The
primary objective is the native actor loss of the underlying algorithm.  The
auxiliary objective is the established consecutive-state consistency penalty
\(
  \mathbb{E}_{(s,s')}
  [\lVert\pi_\theta(s')-\pi_\theta(s)\rVert_2^2]
\), where \(\pi_\theta\) is the deterministic action used for evaluation.
Our contribution is not this temporal loss itself, but a rule for combining its
gradient with the native actor gradient.

We introduce \method{}, which combines the native actor and temporal gradients through two operations (Figure~\ref{fig:damper-overview}).  First,
when the temporal gradient conflicts with the native actor gradient, \method{}
removes only its return-opposing component.  The native actor gradient is never
projected, and an aligned temporal component is retained rather than discarded.
Second, \method{} controls the remaining direction relative to the native
gradient norm.  A bounded parameter $\eta\in[0,1]$ interpolates between two endpoints: at $\eta=0$, the temporal component is only capped at the native gradient norm, and at $\eta=1$, it is rescaled to that norm.  Unlike a scalar
regularization coefficient, \(\eta\) controls how the observed gradient
magnitudes are reconciled after the objectives have been differentiated. Figure~\ref{fig:damper-overview} provides an overview of these two operations.

Across six continuous-control tasks and two backbones, \method{} improves the
FFT-based oscillation score over the corresponding unregularized agent in all
12 task--backbone pairs and achieves the best score among the compared methods
in 8.  The improvement is more consistent for action stability than for return:
several locomotion settings exhibit a clear smoothness--return trade-off.  We
report this limitation explicitly and isolate the role of conflict-conditioned
projection in an ablation study.

Our contributions are as follows:
\begin{itemize}
    \item We introduce DAMPER, a return-prioritized actor update that
    combines conflict-conditioned projection with norm-ratio interpolation.

    \item We establish finite-step descent for the fixed actor surrogate
    and characterize temporal progress at matched first-order primary
    progress. 
    \item DAMPER reduces action oscillation relative to native TD3 and SAC
on all six tasks, achieving the best score in eight of 12
task--backbone pairs. Projection ablations, interpolation sweeps,
and fixed-weight scalarization comparisons assess the contributions
of directional and magnitude control.
\end{itemize}

\section{Related Work}
\label{sec:related_work}

\subsection{Policy Smoothness in Reinforcement Learning}

\paragraph{Policy-output regularization.}
CAPS combines temporal consistency between consecutive states with spatial
consistency under state perturbations~\parencite{mysore2021caps}.  Closely related
work directly penalizes consecutive deterministic policy outputs
\parencite{decooman2021temporal}.  We use the same form of temporal surrogate, but
do not add it to the native actor loss with a fixed coefficient.  Instead, we
modify its gradient according to its alignment and scale relative to the native
actor gradient.

\paragraph{Sensitivity and critic regularization.}
Grad-CAPS regularizes policy sensitivity~\parencite{lee2024gradcaps}, ASAP combines
action alignment, action prediction, and second-order temporal regularization
\parencite{kwak2026asap}, and L2C2 constrains local changes in both the policy and
value function~\parencite{kobayashi2022l2c2}.  PAVE instead regularizes the critic
to stabilize the action-gradient field that drives policy learning
\parencite{lee2026pave}.  These approaches encode richer notions of regularity than
our single temporal loss, but multi-component scalarized objectives introduce
additional relative weights.  \method{} is complementary: it focuses on how a
temporal auxiliary gradient should enter the actor update when the native actor
objective remains primary.

\paragraph{Structured smooth control.}
Smooth behavior can also arise from architectural or action-space structure.
LipsNet controls policy sensitivity through an adaptive Lipschitz architecture
\parencite{song2023lipsnet}; generalized state-dependent exploration produces
temporally correlated exploration noise~\parencite{raffin2022smooth}; and action-rate
parameterizations smooth the executed control by construction
\parencite{chisari2021racing}.  In contrast, \method{} changes only training-time
actor gradients and leaves the policy architecture and execution interface
unchanged.

\subsection{Gradient Combination and Objective Priority}

\paragraph{Multi-objective optimization.}
MGDA searches for a common descent direction
\parencite{desideri2012mgda,sener2018multiobjective}, GradNorm balances task-gradient
magnitudes~\parencite{chen2018gradnorm}, PCGrad removes conflicting gradient
components~\parencite{yu2020pcgrad}, and CAGrad constructs a conflict-averse common
update~\parencite{liu2021cagrad}.  These methods typically treat objectives
symmetrically.  Our setting is hierarchical: the native actor objective is
always primary, so its gradient is never altered to accommodate temporal
consistency.

\paragraph{Primary--auxiliary learning.}
Gradient-similarity weighting~\parencite{du2018auxiliary}, Dynamic
Barrier~\parencite{gong2021dynamicbarrier}, and Bloop~\parencite{hsieh2024bloop} prevent
auxiliary objectives from degrading a primary task, while
MetaBalance~\parencite{he2022metabalance} adapts auxiliary-gradient magnitudes to a
target gradient.  \method{} combines the two relevant ideas: it conditionally
removes only the primary-opposing temporal component and separately controls
its magnitude through a bounded norm-ratio interpolation.

\paragraph{Priority-aware reinforcement learning.}
PEGrad prioritizes reward over energy minimization by orthogonalizing and
capping the lower-priority gradient~\parencite{peri2025pegrad}; GCR-PPO conditionally
projects lower-priority reward components~\parencite{munn2026gcrppo}; and LPPG-RL
considers lexicographically ordered policy-gradient objectives
\parencite{qiu2026lppgrl}.  FCGrad addresses individual--collective return conflicts
in mixed-motive multi-agent learning and changes priority according to current
returns~\parencite{kim2025fcgrad}.  Our hierarchy is fixed, projection occurs only
under conflict, and the retained temporal direction is scaled by an explicit
interpolation between cap-only and norm-balanced updates.  The contribution is
this combination for temporal policy smoothing, rather than gradient
projection in isolation.

\section{Preliminaries}
\label{sec:preliminaries}

\paragraph{Actor--critic learning.}
We consider a discounted Markov decision process
\(\mathcal M=(\mathcal S,\mathcal A,p,r,\gamma)\), with transition
kernel \(p(s'\mid s,a)\), reward \(r(s,a)\), and discount factor
\(\gamma\in[0,1)\). A policy \(\pi_\theta\) seeks to maximize
the expected return
\[
    J(\theta)
    :=\mathbb E_{\pi_\theta}
    \left[\sum_{t=0}^{\infty}\gamma^t r(s_t,a_t)\right].
\]
Actor--critic methods use a learned action-value function
\(Q_\phi(s,a)\) to guide policy optimization. We adopt the
loss-minimization convention and denote the underlying algorithm's
native actor loss by \(\returnloss(\theta)\).
Algorithm-specific definitions are provided in
Appendix~\ref{app:native_actor_objectives}.

\paragraph{Temporal policy consistency.}
Let \(\pi_\theta(s)\) denote the differentiable deterministic
action map associated with the current policy. For consecutive states
\((s,s')\) drawn from an observed transition distribution
\(\mathcal D\), we define the temporal-consistency loss
\begin{equation}
    \temporalloss(\theta)
    :=\mathbb E_{(s,s')\sim\mathcal D}
    \left[
        \left\lVert\pi_\theta(s')-\pi_\theta(s)\right\rVert_2^2
    \right].
    \label{eq:temporal_consistency_loss}
\end{equation}
Both actions are computed using the same current policy.
Minimizing this loss encourages similar actions at consecutive states,
promoting temporal consistency along environment trajectories.

\paragraph{Combining objectives.}
We denote the actor-parameter gradients by
\begin{equation}
    g_R:=\nabla_\theta\returnloss(\theta),
    \qquad
    g_T:=\nabla_\theta\temporalloss(\theta).
    \label{eq:primary_auxiliary_gradients}
\end{equation}
The term ``return gradient'' refers to \(g_R\), the critic-based
actor-loss gradient, rather than a direct gradient of measured
episodic return. A conventional auxiliary-loss formulation minimizes
\begin{equation}
    \mathcal L_\lambda(\theta)
    :=\returnloss(\theta)+\lambda\temporalloss(\theta),
    \qquad \lambda>0,
    \label{eq:weighted_actor_objective}
\end{equation}
which yields the combined gradient \(g_R+\lambda g_T\).


\section{Method}
\label{sec:method}

We introduce \method{}, an actor update rule designed to reduce action
oscillation while prioritizing return maximization. It uses a
temporal-consistency objective to encourage similar actions at consecutive
states, with the native actor objective taking priority. To combine their
gradients, \method{} first removes any temporal component that opposes the
native actor gradient, then controls the remaining contribution relative
to the native gradient norm. These two operations address directional
conflict and scale imbalance when incorporating temporal consistency
into the actor update. Figure~\ref{fig:damper-overview} summarizes the
complete update.

\subsection{Conflict-Conditioned Projection}
\label{sec:method_projection}

We use conflict-conditioned projection to prevent the temporal gradient
from hindering first-order progress on the native actor objective.
Let \(g_R=\nabla_\theta\returnloss\) and
\(g_T=\nabla_\theta\temporalloss\) denote the gradients of the objectives
defined in Section~\ref{sec:preliminaries}.
When \(g_R^\top g_T<0\), the temporal gradient contains a component
that opposes the native actor gradient. We remove this component by
projecting \(g_T\) onto the subspace orthogonal to \(g_R\).
Otherwise, we leave \(g_T\) unchanged.
Denoting the resulting temporal gradient by \(\widetilde g_T\), we obtain
\begin{equation}
    \widetilde g_T
    :=
    \begin{cases}
        g_T-\dfrac{g_R^\top g_T}{\lVert g_R\rVert_2^2}g_R,
        & g_R^\top g_T<0,\\[2mm]
        g_T, & g_R^\top g_T\geq0.
    \end{cases}
    \label{eq:conflict_projection}
\end{equation}
The native gradient \(g_R\) is never modified.
The retained temporal gradient satisfies
\(g_R^\top\widetilde g_T\geq0\), so adding a nonnegative multiple of
\(\widetilde g_T\) preserves the first-order decrease in the native
actor loss provided by \(g_R\).
The Conflict and Aligned panels in Figure~\ref{fig:damper-overview}
illustrate the two cases.

\subsection{Adaptive Magnitude Control}
\label{sec:method_magnitude}

Projection resolves directional conflict, but a temporal gradient that is small relative to the native actor gradient may have little influence on policy smoothness, whereas a much larger one may dominate the actor update.
We therefore rescale the temporal gradient according to the relative norms of the original actor and temporal gradients.

For nonzero gradients, let $n_R:=\lVert g_R\rVert_2$ and $n_T:=\lVert g_T\rVert_2$.
Before accounting for directional conflict, consider the target combined gradient
$$
g_{\text{target}}(s) \coloneqq g_R + s \frac{g_T}{n_T}, \qquad s > 0.
$$
This target adds a temporal contribution of magnitude $s$ along $g_T$ to the unchanged actor gradient $g_R$. 

The target may, however, reduce first-order progress on the native actor objective relative to using $g_R$ alone.
We therefore make the smallest adjustment, measured in Euclidean distance, to $g_{\text{target}}(s)$ that preserves this progress:
$$
g(s) \coloneqq \argmin_{d: g_R^\top d \geq \lVert g_R \rVert_2^2} \Vert d - g_{\text{target}}(s) \Vert_2^2.
$$
The constraint requires at least the first-order decrease in the native actor loss achieved by $g_R$.
It can be shown (Appendix~\ref{app:gradient_combination}) that the unique solution to the optimization problem is
$$
g(s) = g_R + \frac{s}{n_T} \widetilde g_T,
$$
where $\widetilde g_T$ is exactly the conflict-conditioned projection from the previous subsection.

We consider two reference target magnitudes.
The \textit{cap-only} choice $s_0 \coloneqq \min(n_T, n_R)$ caps the target temporal magnitude at the native actor gradient norm $n_R$, without amplifying weak temporal gradients.
The \textit{norm-matched} choice $s_1 \coloneqq n_R$ instead rescales the temporal contribution to $n_R$ before projection, amplifying it when $n_T < n_R$.

To obtain intermediate behavior, we geometrically interpolate between these choices using a parameter $\eta \in [0, 1]$:
\begin{equation}
s_\eta \coloneqq s_0^{1 - \eta} s_1^\eta, \qquad g_\eta \coloneqq g(s_\eta) = g_R + \frac{s_\eta}{n_T} \widetilde g_T.
\label{eq:merged_gradient}
\end{equation}
Equivalently, defining $q \coloneqq \min(n_T/n_R, 1)$ and $w_\eta \coloneqq q^{1-\eta}$ gives $s_\eta = n_R w_\eta$.
For $n_T < n_R$, increasing $\eta$ strengthens the temporal gradient.

\subsection{DAMPER Actor Update}
\label{sec:method_integration}

Algorithm~\ref{alg:damper_update} summarizes the core actor update
for nonzero gradients. \method{} computes the two gradients separately,
applies conflict-conditioned projection and adaptive magnitude control,
and passes the merged gradient to the native actor optimizer.

\begin{algorithm}[t]
\caption{\method{} actor update}
\label{alg:damper_update}
\small
\DontPrintSemicolon

\KwIn{Actor parameters $\theta$, replay buffer $\mathcal D$,
      native actor loss $\returnloss$, interpolation $\eta\in[0,1]$.}

Sample a minibatch $\mathcal B$ from $\mathcal D$\;
Evaluate $\returnloss(\theta)$ and
$\temporalloss(\theta)$ in \eqref{eq:temporal_consistency_loss}\;

$g_R\gets\nabla_\theta\returnloss$,
$g_T\gets\nabla_\theta\temporalloss$\;
$n_R\gets\lVert g_R\rVert_2$,
$n_T\gets\lVert g_T\rVert_2$\;

$\widetilde g_T\gets g_T$\;
\If{$g_R^\top g_T<0$}{
    $\widetilde g_T\gets
    g_T-\dfrac{g_R^\top g_T}{n_R^2}g_R$\;
}

$q\gets\min(n_T/n_R,1)$,
$w_\eta\gets q^{1-\eta}$\;
$g\gets g_R+(n_R/n_T)w_\eta\widetilde g_T$\;

Assign $g$ to the actor-gradient buffers and apply the native
actor optimizer\;
\end{algorithm}

\section{Analysis of the Primary--Auxiliary Update}
\label{sec:analysis}

We analyze a single Euclidean actor step using the direction
constructed by Algorithm~\ref{alg:damper_update}, with the critic
and sampled transitions held fixed.

\subsection{Descent on the Native Actor Loss}

Consistent with our primary--auxiliary formuation, in which the native actor objective tasks priority over temporal consistency, we first show that \method{} preserves descent on the native actor loss even when the two gradients conflicts.

\begin{theorem}[Finite-step descent for the fixed actor surrogate]
\label{thm:finite_actor_descent}
Assume that \(\returnloss\) is differentiable and globally
\(\beta_R\)-smooth for \(\beta_R>0\), and that
\(g_R=\nabla_\theta\returnloss(\theta)\neq0\).
Let \(g\) be the direction returned by
Algorithm~\ref{alg:damper_update}. Then every step
\begin{equation}
    \theta^+=\theta-\alpha g,
    \qquad
    0<\alpha<\frac{1}{\beta_R},
    \label{eq:finite_descent_step}
\end{equation}
strictly decreases the fixed native actor loss:
\(\returnloss(\theta^+)<\returnloss(\theta)\).
\end{theorem}

Thus, for sufficiently small steps, \method{} incorporates the temporal
gradient while preserving descent on the native actor loss.
The proof is in Appendix~\ref{app:proof_descent}.

\subsection{Temporal Progress at Matched Primary Progress}
\label{sec:matched_progress}

We next compare the temporal progress of DAMPER and fixed-weight scalarization.
For this comparison, we rescale each update to give the same predicted decrease in the native actor loss, using a first-order approximation.
This holds primary improvement fixed, allowing us to compare the temporal benefit of the update directions independently of their overall scale.
For any direction $d$ with $g_R^\top d>0$, define
\begin{equation}
\mathcal G(d) := \frac{g_T^\top d}{g_R^\top d} - \frac{g_T^\top g_R}{\lVert g_R\rVert_2^2}.
\label{eq:matched_temporal_gain}
\end{equation}
The first term measures the predicted temporal decrease per unit of predicted native-actor-loss decrease; subtracting the corresponding quantity for the native-gradient update gives the additional temporal gain.

Using this measure, the next result shows that, for noncollinear gradients and fixed $\eta$, DAMPER achieves greater additional temporal gain than scalarization with any fixed positive weight when the temporal gradient is sufficiently small relative to the native actor gradient.

\begin{theorem}[Temporal gain in the weak-gradient regime]
\label{thm:weak_gradient_advantage}
Let \(g_R=\nabla_\theta L_R(\theta)\) and
\(g_T=\nabla_\theta L_T(\theta)\) be nonzero, and define
\[
    q:=\min\left\{
        \frac{\lVert g_T\rVert_2}{\lVert g_R\rVert_2},1
    \right\},
    \qquad
    c:=\frac{g_R^\top g_T}
    {\lVert g_R\rVert_2\lVert g_T\rVert_2}.
\]
Fix \(\eta\in[0,1]\) and any \(\lambda>0\).
Assume \(0<q<1\) and \(\lambda q<1-\delta\)
for some fixed \(\delta\in(0,1)\).
For the analytic \method{} direction \(g_\eta\) in
\eqref{eq:merged_gradient} and
\(g_\lambda=g_R+\lambda g_T\), the matched temporal gains satisfy
\begin{equation}
    \mathcal G(g_\eta)=\Theta\!\left((1-c^2)q^{2-\eta}\right),
    \qquad
    \mathcal G(g_\lambda)=\Theta\!\left((1-c^2)q^2\right).
    \label{eq:weak_gradient_rates}
\end{equation}
Here, \(\Theta\) denotes two-sided bounds with positive constants
independent of \(q\) and \(c\), but possibly dependent on
\(\lambda\) and \(\delta\).
\end{theorem}

The result highlights the role of magnitude control in the weak-temporal regime, where the temporal gradient is small relative to the native actor gradient.
For fixed $\eta > 0$ and $\lambda > 0$, DAMPER provides greater additional temporal gain than scalarization and sufficiently small $q$, when the first-order primary progress is matched.
Appendix~\ref{app:proof_weak_gradient} gives the proof.
\section{Experiments}
\label{sec:experiments}
\subsection{Experimental Setup}
\label{sec:experimental_setup}

In this subsection, we describe the tasks, baselines, and evaluation protocol
used to assess \method{}.

\paragraph{Tasks and backbones.}
We consider six continuous-control tasks from Gymnasium
\parencite{towers2024gymnasium} and MuJoCo~\parencite{todorov2012mujoco}:
LunarLander, Pendulum, Reacher, Ant, Hopper, and Walker.  Each method is
evaluated with both TD3 and SAC.  Our implementation and protocol follow the
PAVE codebase, and all neural networks use SiLU activations.

\paragraph{Baselines.}
We compare with the unregularized TD3 and SAC agents and five smooth-control
methods: CAPS, Grad-CAPS, ASAP, L2C2, and PAVE.  The comparison therefore spans
policy-output regularization, policy-sensitivity regularization, joint
policy--value constraints, and critic-gradient stabilization.

\paragraph{Evaluation protocol.}
Each method is trained with five independent random seeds. For each seed,
the learned policy is evaluated for ten episodes; tables report the mean
and standard deviation across seeds. We measure cumulative episodic return
(\(re\), higher is better) and action oscillation
(\(sm\), lower is better). The latter is a spectral smoothness metric
computed using the fast Fourier transform (FFT), following
\textcite{mysore2021caps,christmann2024benchmarking}.


\subsection{Overall Performance}
\label{sec:experimental_results}

In this subsection, we show that \method{} consistently reduces action
oscillation across TD3 and SAC, with task-dependent effects on return.
It lowers \(sm\) relative to the unregularized agents in all 12 task--backbone
pairs and achieves the lowest score among the compared methods in eight
(Tables~\ref{tab:td3-results} and~\ref{tab:sac-results}).
Figure~\ref{fig:action-trajectory} also shows smaller consecutive action
changes on Walker for both backbones.

\begin{figure}[!htbp]
    \centering
    \includegraphics[width=\linewidth]
    {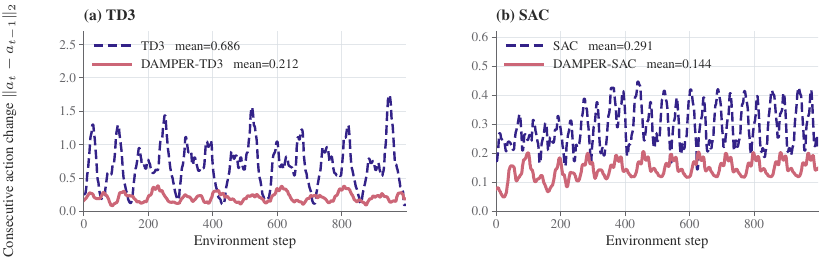}
    \caption{\textbf{\method{} reduces consecutive action variation
    on Walker under both actor--critic backbones.}
    We plot \(\lVert a_t-a_{t-1}\rVert_2\) over a 1,000-step rollout.}
    \label{fig:action-trajectory}

\end{figure}
\begin{table}[t]
\caption{\textbf{TD3 results.} Entries are mean (standard deviation) across
five seeds. Higher return ($re$) and lower oscillation score ($sm$) are
better. The best value in each column is highlighted in \textbf{bold}.}
\label{tab:td3-results}
\centering
\providecommand{\meanstd}[2]{\shortstack[c]{#1\\[-0.15em]{\scriptsize(#2)}}}
\setlength{\tabcolsep}{4pt}
\renewcommand{\arraystretch}{1.25}
\resizebox{\linewidth}{!}{%
\begin{tabular}{@{}l*{12}{c}@{}}
\toprule
& \multicolumn{2}{c}{LunarLander}
& \multicolumn{2}{c}{Pendulum}
& \multicolumn{2}{c}{Reacher}
& \multicolumn{2}{c}{Ant}
& \multicolumn{2}{c}{Hopper}
& \multicolumn{2}{c}{Walker} \\
\cmidrule(lr){2-3}\cmidrule(lr){4-5}\cmidrule(lr){6-7}%
\cmidrule(lr){8-9}\cmidrule(lr){10-11}\cmidrule(lr){12-13}
Method
& $re\!\uparrow$ & $sm\!\downarrow$
& $re\!\uparrow$ & $sm\!\downarrow$
& $re\!\uparrow$ & $sm\!\downarrow$
& $re\!\uparrow$ & $sm\!\downarrow$
& $re\!\uparrow$ & $sm\!\downarrow$
& $re\!\uparrow$ & $sm\!\downarrow$ \\
\midrule
TD3
& \meanstd{205.5}{98.5}  & \meanstd{1.815}{0.955}
& \meanstd{$-$157.6}{84.0} & \meanstd{2.308}{1.422}
& \meanstd{$\mathbf{-3.34}$}{1.45} & \meanstd{0.050}{0.015}
& \meanstd{4749}{1226} & \meanstd{2.044}{0.307}
& \meanstd{\textbf{3604}}{150} & \meanstd{3.061}{0.421}
& \meanstd{4596}{328} & \meanstd{1.937}{0.139} \\
CAPS
& \meanstd{246.1}{81.6} & \meanstd{0.634}{0.233}
& \meanstd{$-$159.4}{81.7} & \meanstd{0.454}{0.140}
& \meanstd{$\mathbf{-3.34}$}{1.47} & \meanstd{0.046}{0.015}
& \meanstd{\textbf{5370}}{275} & \meanstd{2.064}{0.101}
& \meanstd{2928}{1025} & \meanstd{1.730}{0.341}
& \meanstd{5032}{241} & \meanstd{1.865}{0.176} \\
Grad-CAPS
& \meanstd{225.5}{61.4} & \meanstd{0.905}{0.333}
& \meanstd{$\mathbf{-153.4}$}{77.9} & \meanstd{1.011}{0.503}
& \meanstd{$\mathbf{-3.34}$}{1.49} & \meanstd{\textbf{0.040}}{0.011}
& \meanstd{4201}{1859} & \meanstd{1.713}{0.408}
& \meanstd{3371}{619} & \meanstd{1.287}{0.128}
& \meanstd{\textbf{5046}}{460} & \meanstd{1.412}{0.291} \\
ASAP
& \meanstd{245.5}{58.6} & \meanstd{1.345}{0.616}
& \meanstd{$-$155.3}{78.3} & \meanstd{2.013}{1.047}
& \meanstd{$-$3.36}{1.44} & \meanstd{0.048}{0.015}
& \meanstd{4925}{1210} & \meanstd{2.041}{0.343}
& \meanstd{2581}{1448} & \meanstd{1.866}{0.883}
& \meanstd{4880}{745} & \meanstd{1.615}{0.288} \\
L2C2
& \meanstd{218.5}{83.3} & \meanstd{1.626}{0.727}
& \meanstd{$-$156.0}{79.1} & \meanstd{3.293}{1.034}
& \meanstd{$-$3.38}{1.52} & \meanstd{0.048}{0.014}
& \meanstd{4078}{1658} & \meanstd{1.974}{0.402}
& \meanstd{2807}{1371} & \meanstd{2.427}{1.087}
& \meanstd{4960}{1048} & \meanstd{1.707}{0.320} \\
PAVE
& \meanstd{238.1}{71.5} & \meanstd{0.841}{0.601}
& \meanstd{$-$154.5}{80.4} & \meanstd{\textbf{0.410}}{0.121}
& \meanstd{$-$3.55}{1.51} & \meanstd{0.046}{0.016}
& \meanstd{5064}{1245} & \meanstd{1.865}{0.237}
& \meanstd{3328}{541} & \meanstd{1.022}{0.148}
& \meanstd{4996}{700} & \meanstd{1.991}{0.410} \\
\methodplain{} (ours)
& \meanstd{\textbf{266.4}}{35.8} & \meanstd{\textbf{0.313}}{0.036}
& \meanstd{$-$154.6}{78.7} & \meanstd{0.463}{0.131}
& \meanstd{$-$3.63}{1.44} & \meanstd{0.048}{0.016}
& \meanstd{5281}{887} & \meanstd{\textbf{1.261}}{0.147}
& \meanstd{3078}{599} & \meanstd{\textbf{0.240}}{0.035}
& \meanstd{4275}{322} & \meanstd{\textbf{0.346}}{0.088} \\
\bottomrule
\end{tabular}%
}
\end{table}

\paragraph{TD3 results.}
\method{} achieves the lowest oscillation score on LunarLander, Ant, Hopper,
and Walker (Table~\ref{tab:td3-results}). On LunarLander, it improves both
metrics, increasing return from \(205.5\) to \(266.4\) and reducing \(sm\)
from \(1.815\) to \(0.313\). Hopper and Walker also show large reductions
in \(sm\), from \(3.061\) to \(0.240\) and from \(1.937\) to \(0.346\),
respectively, accompanied by lower returns than unregularized TD3.

\begin{table}[t]
\caption{\textbf{SAC results.} Entries are mean (standard deviation) across
five seeds. Higher return ($re$) and lower oscillation score ($sm$) are
better. The best value in each column is highlighted in \textbf{bold}.}
\label{tab:sac-results}
\centering
\providecommand{\meanstd}[2]{\shortstack[c]{#1\\[-0.15em]{\scriptsize(#2)}}}
\setlength{\tabcolsep}{4pt}
\renewcommand{\arraystretch}{1.25}
\resizebox{\linewidth}{!}{%
\begin{tabular}{@{}l*{12}{c}@{}}
\toprule
& \multicolumn{2}{c}{LunarLander}
& \multicolumn{2}{c}{Pendulum}
& \multicolumn{2}{c}{Reacher}
& \multicolumn{2}{c}{Ant}
& \multicolumn{2}{c}{Hopper}
& \multicolumn{2}{c}{Walker} \\
\cmidrule(lr){2-3}\cmidrule(lr){4-5}\cmidrule(lr){6-7}%
\cmidrule(lr){8-9}\cmidrule(lr){10-11}\cmidrule(lr){12-13}
Method
& $re\!\uparrow$ & $sm\!\downarrow$
& $re\!\uparrow$ & $sm\!\downarrow$
& $re\!\uparrow$ & $sm\!\downarrow$
& $re\!\uparrow$ & $sm\!\downarrow$
& $re\!\uparrow$ & $sm\!\downarrow$
& $re\!\uparrow$ & $sm\!\downarrow$ \\
\midrule
SAC
& \meanstd{162.6}{139.9} & \meanstd{0.402}{0.143}
& \meanstd{$-$149.9}{76.9} & \meanstd{0.479}{0.128}
& \meanstd{$-$3.49}{1.43} & \meanstd{0.053}{0.016}
& \meanstd{5001}{1145} & \meanstd{1.935}{0.234}
& \meanstd{3059}{825} & \meanstd{0.708}{0.074}
& \meanstd{4771}{292} & \meanstd{0.747}{0.081} \\
CAPS
& \meanstd{120.0}{171.6} & \meanstd{0.279}{0.065}
& \meanstd{$-$152.7}{80.1} & \meanstd{\textbf{0.325}}{0.107}
& \meanstd{$-$3.48}{1.39} & \meanstd{0.048}{0.014}
& \meanstd{5439}{707} & \meanstd{1.931}{0.170}
& \meanstd{3222}{548} & \meanstd{0.653}{0.058}
& \meanstd{4955}{296} & \meanstd{0.745}{0.164} \\
Grad-CAPS
& \meanstd{269.0}{27.6} & \meanstd{0.277}{0.064}
& \meanstd{$\mathbf{-148.9}$}{75.7} & \meanstd{0.379}{0.119}
& \meanstd{$-$3.51}{1.41} & \meanstd{0.046}{0.013}
& \meanstd{\textbf{5675}}{960} & \meanstd{1.797}{0.248}
& \meanstd{3144}{653} & \meanstd{\textbf{0.494}}{0.061}
& \meanstd{\textbf{5060}}{262} & \meanstd{0.585}{0.067} \\
ASAP
& \meanstd{270.1}{19.6} & \meanstd{0.303}{0.080}
& \meanstd{$-$149.2}{76.0} & \meanstd{0.505}{0.138}
& \meanstd{$-$3.46}{1.40} & \meanstd{0.053}{0.016}
& \meanstd{4564}{2104} & \meanstd{1.727}{0.326}
& \meanstd{2831}{908} & \meanstd{0.637}{0.104}
& \meanstd{5035}{408} & \meanstd{0.760}{0.155} \\
L2C2
& \meanstd{239.4}{58.2} & \meanstd{0.407}{0.168}
& \meanstd{$-$149.3}{76.2} & \meanstd{0.510}{0.145}
& \meanstd{$\mathbf{-3.45}$}{1.38} & \meanstd{0.052}{0.016}
& \meanstd{4947}{1565} & \meanstd{1.810}{0.415}
& \meanstd{3101}{580} & \meanstd{0.732}{0.106}
& \meanstd{4834}{277} & \meanstd{0.901}{0.221} \\
PAVE
& \meanstd{263.6}{24.2} & \meanstd{0.203}{0.034}
& \meanstd{$-$151.6}{78.7} & \meanstd{0.334}{0.114}
& \meanstd{$-$3.51}{1.49} & \meanstd{0.051}{0.015}
& \meanstd{5543}{1007} & \meanstd{1.741}{0.218}
& \meanstd{2403}{862} & \meanstd{0.584}{0.097}
& \meanstd{4544}{646} & \meanstd{0.713}{0.090} \\
\methodplain{} (ours)
& \meanstd{\textbf{281.5}}{24.4} & \meanstd{\textbf{0.102}}{0.037}
& \meanstd{$-$156.6}{78.5} & \meanstd{0.336}{0.164}
& \meanstd{$-$3.57}{1.57} & \meanstd{\textbf{0.038}}{0.018}
& \meanstd{4564}{1577} & \meanstd{\textbf{1.207}}{0.253}
& \meanstd{\textbf{3384}}{273} & \meanstd{0.571}{0.022}
& \meanstd{4402}{1086} & \meanstd{\textbf{0.268}}{0.044} \\
\bottomrule
\end{tabular}%
}
\end{table}

\paragraph{SAC results.}
\method{} achieves the lowest oscillation score on LunarLander, Reacher, Ant,
and Walker (Table~\ref{tab:sac-results}). It also obtains the highest return
on LunarLander and Hopper; on Hopper, \(sm\) decreases from \(0.708\) to
\(0.571\) relative to unregularized SAC. On Ant and Walker, the lowest
oscillation scores are accompanied by lower returns than the unregularized
agent, again showing that the return--smoothness balance depends on the task.

\subsection{Ablation Studies}
\label{sec:ablation}

\paragraph{Projection Ablation.}We show that conflict-conditioned projection achieves
the highest mean return among the evaluated update rules on three TD3 tasks:
LunarLander, Hopper, and Walker. We compare \method{} with
\emph{Unconditional Sum}, which always adds the scaled temporal direction,
and \emph{Unconditional Projection}, which orthogonalizes it at every update.
The former retains components that oppose the native actor gradient, while
the latter removes aligned components as well. 

\begin{table}[!htbp]
\small
\caption{\textbf{Ablation of conflict-conditioned projection with
TD3.} Entries are mean (standard deviation); each variant uses five training
seeds and ten evaluation episodes per seed.  The numerically best mean for each
metric is shown in bold.}
\label{tab:projection-ablation}

\centering
\setlength{\tabcolsep}{4pt}
\resizebox{\linewidth}{!}{%
\begin{tabular}{lcccccc}
\toprule
& \multicolumn{2}{c}{LunarLander}
& \multicolumn{2}{c}{Hopper}
& \multicolumn{2}{c}{Walker} \\
\cmidrule(lr){2-3}\cmidrule(lr){4-5}\cmidrule(lr){6-7}
Method
& \(re\!\uparrow\) & \(sm\!\downarrow\)
& \(re\!\uparrow\) & \(sm\!\downarrow\)
& \(re\!\uparrow\) & \(sm\!\downarrow\) \\
\midrule
\methodplain{}
& \textbf{266.4} {\scriptsize(35.8)} & 0.313 {\scriptsize(0.036)}
& \textbf{3078} {\scriptsize(599)} & 0.240 {\scriptsize(0.035)}
& \textbf{4275} {\scriptsize(322)} & 0.346 {\scriptsize(0.088)} \\
Unconditional projection
& 206.9 {\scriptsize(122.7)} & 0.370 {\scriptsize(0.223)}
& 3050 {\scriptsize(587)} & 0.249 {\scriptsize(0.040)}
& 3525 {\scriptsize(1576)} & 0.314 {\scriptsize(0.072)} \\
Unconditional sum
& 263.5 {\scriptsize(68.8)} & \textbf{0.279} {\scriptsize(0.113)}
& 317 {\scriptsize(341)} & \textbf{0.053} {\scriptsize(0.106)}
& 3575 {\scriptsize(514)} & \textbf{0.268} {\scriptsize(0.088)} \\
\bottomrule
\end{tabular}
}
\end{table}

On Hopper, unconditional summation reduces mean return from \(3078\) to
\(317\), despite achieving a lower oscillation score
(Table~\ref{tab:projection-ablation}). This result highlights the need to
assess smoothness together with task performance. Unconditional projection
yields a similar Hopper return but lowers Walker return from \(4275\) to
\(3525\), with greater variability. On LunarLander, unconditional summation
remains close to \method{} in mean return, whereas unconditional projection
has a lower mean and greater variability. These results support resolving
conflicting components while retaining aligned temporal information.

\begin{table}[!htbp]
\small
\caption{\textbf{PEGrad-based baseline with TD3.}
Mean (standard deviation); five training seeds and ten evaluation
episodes per seed.}
\label{tab:pegrad-comparison}
\centering
\setlength{\tabcolsep}{3pt}
\begin{tabular}{@{}lcc@{}}
\toprule
Task & \(re\!\uparrow\) & \(sm\!\downarrow\) \\
\midrule
LunarLander & 198.8 {\scriptsize(123.5)} & 0.619 {\scriptsize(0.353)} \\
Hopper      & 2760 {\scriptsize(835)}   & 0.636 {\scriptsize(0.046)} \\
Walker      & 4361 {\scriptsize(176)}   & 0.712 {\scriptsize(0.142)} \\
\bottomrule
\end{tabular}
\end{table}

\paragraph{Comparison with a PEGrad-Based Baseline.}
We compare \method{} with a PEGrad-based baseline~\parencite{peri2025pegrad}. This baseline doesn't use conditional projection and limits the temporal gradient's magnitude, but lacks a parameter for strengthening weak temporal gradients relative to the native actor gradient. \method{} can amplify weak temporal gradients through
\(w_\eta=q^{1-\eta}\). Across the three tasks, \method{} lowers the
mean oscillation score by \(49\)--\(62\%\)
(Tables~\ref{tab:projection-ablation} and~\ref{tab:pegrad-comparison}),
while achieving higher mean returns on LunarLander and Hopper and a
slightly lower mean return on Walker. These results are consistent with
the benefit of magnitude interpolation alongside conflict-conditioned
projection.

\subsection{Comparison with Fixed-Weight Scalarization}
\label{sec:scalarization_sweep}

We evaluate fixed-weight scalarization on all six TD3 tasks. For each of
the 20 weights \(\lambda\in\Lambda:=\{0.05,0.10,\ldots,1.0\}\), we train
agents using
\begin{equation}
    \mathcal L_{\lambda}(\theta)
    :=
    \returnloss(\theta)
    +\lambda\temporalloss(\theta),
    \qquad \lambda\in\Lambda,
    \label{eq:fixed_scalarization_sweep}
\end{equation}
with \(\lambda\) fixed throughout training.
Figure~\ref{fig:lambda_scalarization_sweep} shows the Hopper comparison
with \method{} at \(\eta=1\).

\begin{figure}[t]
    \centering
    \includegraphics[width=0.92\linewidth]
        {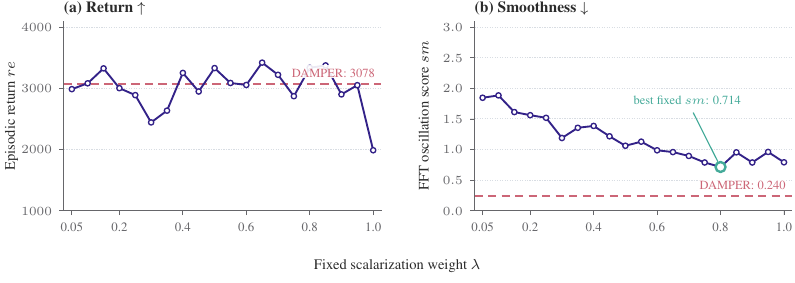}
    \caption{\textbf{Fixed-weight scalarization on TD3 Hopper.}
    Points show means for \(\lambda=0.05,0.10,\ldots,1.0\);
    dashed lines mark \method{}.}
    \label{fig:lambda_scalarization_sweep}
\end{figure}

Across the six tasks, \method{} achieves a lower mean oscillation score
than every evaluated fixed weight on five tasks, with Reacher as the
exception. On Hopper, even the smoothest scalarized policy,
obtained at \(\lambda=0.8\), has \(sm=0.714\), compared with
\(0.240\) for \method{}.
These results are consistent with the temporal-progress analysis
in Theorem~\ref{thm:weak_gradient_advantage}.

Tuned scalarization nevertheless achieves competitive returns and lower
oscillation than several specialized smoothness baselines. On LunarLander,
Ant, and Hopper, selected weights improve both mean metrics over multiple
such baselines, suggesting limited additional benefit from their more
elaborate mechanisms over a well-tuned single temporal loss in these
settings (Appendix~\ref{app:scalarization_sweeps}).

\subsection{Effect of the Interpolation Parameter}
\label{sec:eta_ablation}

\begin{figure}[!htbp]
    \centering
    \includegraphics[width=0.65\linewidth]
        {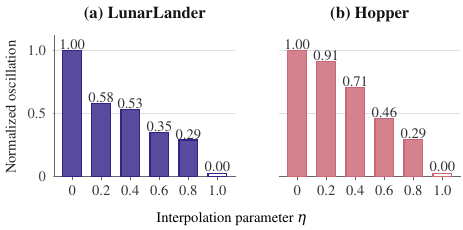}
    \hfill
    \includegraphics[width=0.334\linewidth]
        {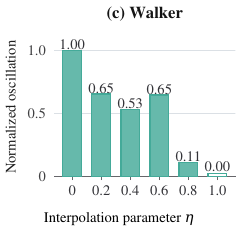}
    \setlength{\abovecaptionskip}{3pt}
    \setlength{\belowcaptionskip}{0pt}
    \small
    \caption{Effect of \(\eta\) on task-normalized action
    oscillation with TD3.}
    \label{fig:eta_smoothness_ablation}
\end{figure}
We sweep \(\eta\in\mathcal H=\{0,0.2,\ldots,1.0\}\) on LunarLander,
Hopper, and Walker using TD3. For each task \(e\), we normalize
the mean oscillation score over the sweep as
\begin{equation}
    \widetilde{sm}_{e}(\eta)
    =\frac{sm_{e}(\eta)-m_e}{M_e-m_e},
    \label{eq:normalized_eta_smoothness}
\end{equation}
where \(m_e\) and \(M_e\) are the minimum and maximum mean scores
over \(\mathcal H\), respectively. Zero therefore denotes the
smoothest tested setting for each task; absolute scores remain
task-dependent.

Figure~\ref{fig:eta_smoothness_ablation} shows that oscillation
decreases monotonically on LunarLander and Hopper, while Walker
exhibits a temporary increase at \(\eta=0.6\). All three tasks
attain their lowest mean score at \(\eta=1\). This trend is
consistent with Theorem~\ref{thm:weak_gradient_advantage}:
increasing \(\eta\) makes \(\mathcal G(g_\eta)\) nondecreasing
at matched first-order primary progress.

\section{Conclusion}
\label{sec:conclusion}

We introduced \method{}, a primary--auxiliary actor update that combines
conflict-conditioned temporal-gradient projection with norm-ratio magnitude
control. The update leaves the native actor gradient unchanged, admits a local
descent guarantee for a fixed critic and minibatch, and reduces the reported
action-oscillation score across all 12 evaluated task--backbone pairs.
Return is not uniformly preserved, however, and our theoretical analysis
is limited to a single actor update rather than the full training process.
Evaluation on physical systems is an important next step for understanding
when smoother control can be obtained without sacrificing task performance.

\section*{AI Use Statement}

In this work, we used generative AI tools to assist with writing and coding.
For writing, AI tools were used to edit and polish author-written text for
grammar, clarity, and readability, and to help format \LaTeX{} (e.g., tables
and references). For coding, AI tools were used to assist in creating and
editing software code, including debugging, refactoring, and writing
experiment and plotting scripts.

Generative AI tools were also used to assist with checking and refining mathematical derivations and proofs, including identifying gaps and assessing consistency between theorem statements, assumptions, and supporting arguments. These tools served as supplementary aids, with the authors retaining responsibility for the validity of all theoretical claims and proofs. AI tools were not used to originate research hypotheses or design the research methodology or experiments. Synthetic dataset generation, translation, dataset cleaning, and qualitative data analysis are not applicable to this work.

We have reviewed all AI-assisted work. All AI-edited text was checked by the
authors to ensure that it accurately reflects our intended meaning and claims.
All AI-assisted code was reviewed, tested, and verified for correctness by the
authors before being used to produce the reported results. We take
responsibility for the final content of this work, including text, claims, and
artifacts produced with the aid of generative AI.

\FloatBarrier
\section*{References}
\printbibliography[heading=none]

\appendix
\clearpage
\section{Optimization Characterization of the Combined Gradient}
\label{app:gradient_combination}

The following lemma establishes the closed-form solution to the
optimization problem in Section~4.2.

\begin{lemma}[Closest primary-preserving combined gradient]
\label{lem:primary_preserving_combination}
Let $g_R,g_T\in\mathbb{R}^p$ be nonzero gradients, with
$n_R:=\lVert g_R\rVert_2$ and $n_T:=\lVert g_T\rVert_2$.
For any $s>0$, define the target combined gradient
\[
g_{\mathrm{target}}(s):=g_R+s\frac{g_T}{n_T}.
\]
Then the optimization problem
\begin{equation}
\label{eq:app_combination_problem}
g(s):=
\underset{d\in\mathbb{R}^p:\,g_R^\top d\ge n_R^2}
{\operatorname{arg\,min}}
\left\lVert d-g_{\mathrm{target}}(s)\right\rVert_2^2
\end{equation}
has the unique solution
\begin{equation}
\label{eq:app_combination_solution}
g(s)=g_R+\frac{s}{n_T}\widetilde g_T,
\qquad
\widetilde g_T
:=
g_T-\frac{\min\{g_R^\top g_T,0\}}{n_R^2}g_R.
\end{equation}
Here, $\widetilde g_T$ is the conflict-conditioned projection
defined in Section~4.1.
\end{lemma}

\begin{proof}
Write $t:=g_{\mathrm{target}}(s)$ and $b:=g_R^\top g_T$.

If $b\ge0$, then
\[
g_R^\top t=n_R^2+\frac{s}{n_T}b\ge n_R^2.
\]
Thus, $t$ is feasible and attains an objective value of zero,
making it the unique minimizer. Since $\widetilde g_T=g_T$
in this case, the claimed solution follows.

If $b<0$, define
\[
d^\star
:=
t-\frac{s b}{n_T n_R^2}g_R.
\]
Then $g_R^\top d^\star=n_R^2$, so $d^\star$ is feasible.
Moreover, $d^\star-t$ is a positive multiple of $g_R$.
For any feasible $d$, we have
$g_R^\top(d-d^\star)\ge0$, and hence
\[
\begin{aligned}
\lVert d-t\rVert_2^2-\lVert d^\star-t\rVert_2^2
&=
\lVert d-d^\star\rVert_2^2
+2(d-d^\star)^\top(d^\star-t)\\
&\ge
\lVert d-d^\star\rVert_2^2.
\end{aligned}
\]
The right-hand side is strictly positive unless $d=d^\star$.
Therefore, $d^\star$ is the unique minimizer. Substituting
the definition of $t$ gives
\[
d^\star
=
g_R+\frac{s}{n_T}
\left(g_T-\frac{b}{n_R^2}g_R\right)
=
g_R+\frac{s}{n_T}\widetilde g_T,
\]
as required.
\end{proof}

\paragraph{Connection to magnitude control.}
The cap-only and norm-matched target magnitudes are
$s_0:=\min(n_T,n_R)$ and $s_1:=n_R$, respectively.
For $\eta\in[0,1]$, define
\[
s_\eta:=s_0^{\,1-\eta}s_1^{\,\eta},
\qquad
q:=\min\!\left(\frac{n_T}{n_R},1\right),
\qquad
w_\eta:=q^{\,1-\eta}.
\]
Since $s_0=n_Rq$ and $s_1=n_R$, we obtain
$s_\eta=n_Rw_\eta$. Lemma~\ref{lem:primary_preserving_combination}
therefore yields
\begin{equation}
\label{eq:app_interpolation_equivalence}
g_\eta
=
g(s_\eta)
=
g_R+\frac{s_\eta}{n_T}\widetilde g_T
=
g_R+\frac{n_R}{n_T}w_\eta\widetilde g_T,
\end{equation}
recovering the norm-ratio combination rule.

Because $\lVert\widetilde g_T\rVert_2\le n_T$ and
$s_\eta\le n_R$, the temporal contribution satisfies
\begin{equation}
\label{eq:app_auxiliary_norm_bound}
\lVert g_\eta-g_R\rVert_2
=
s_\eta\frac{\lVert\widetilde g_T\rVert_2}{n_T}
\le s_\eta\le n_R.
\end{equation}
Thus, $s_\eta$ specifies the temporal magnitude before projection,
not necessarily the magnitude of the final temporal contribution.
Under conflict, $\lVert\widetilde g_T\rVert_2<n_T$, so this
contribution is strictly smaller than $s_\eta$.
Using the original norm $n_T$ in the denominator preserves
this attenuation rather than amplifying the projected gradient
back to the target magnitude.

\section{Proofs for the Primary--Auxiliary Analysis}
\label{app:theory}

\subsection{Proof of Theorem~\ref{thm:finite_actor_descent}}
\label{app:proof_descent}

\begin{proof}
Write \(r:=g_R\) and \(n_R:=\lVert r\rVert_2>0\).
If Algorithm~\ref{alg:damper_update} activates its numerical
guard, then \(g=r\). The descent lemma gives
\begin{equation}
\begin{split}
    \returnloss(\theta-\alpha r)
    &\leq
    \returnloss(\theta)
    -\alpha\lVert r\rVert_2^2
    +\frac{\beta_R\alpha^2}{2}\lVert r\rVert_2^2\\
    &=
    \returnloss(\theta)
    -\alpha\left(1-\frac{\beta_R\alpha}{2}\right)n_R^2.
\end{split}
\label{eq:fallback_descent}
\end{equation}
This is strictly smaller than \(\returnloss(\theta)\) whenever
\(0<\alpha<2/\beta_R\), and hence under the stated condition
\(0<\alpha<1/\beta_R\).

Otherwise, \(g=g_\eta\). Define
\(u_R:=g_R/n_R\) and
\(\widetilde u_T:=\widetilde g_T/n_T\),
where \(n_T:=\lVert g_T\rVert_2\).
Let
\begin{equation}
    c_+:=\max\{c,0\},
    \qquad
    c_-:=\max\{-c,0\}.
    \label{eq:c_plus_minus}
\end{equation}
From \eqref{eq:conflict_projection},
\begin{equation}
    u_R^\top\widetilde u_T=c_+,
    \qquad
    \lVert\widetilde u_T\rVert_2^2=1-c_-^2.
    \label{eq:projection_properties_proof}
\end{equation}
Using \(r=n_Ru_R\) and \eqref{eq:merged_gradient},
\begin{align}
    r^\top g_\eta
    &=
    n_R^2u_R^\top
    \left(u_R+w_\eta\widetilde u_T\right)
    \nonumber\\
    &=
    n_R^2\left(1+w_\eta c_+\right)>0.
    \label{eq:return_alignment_proof}
\end{align}
Moreover,
\begin{align}
    \lVert g_\eta\rVert_2^2
    &=
    n_R^2
    \left\lVert u_R+w_\eta\widetilde u_T\right\rVert_2^2
    \nonumber\\
    &=
    n_R^2
    \left[1+2w_\eta c_+
    +w_\eta^2(1-c_-^2)\right].
    \label{eq:merged_norm_proof}
\end{align}
The descent lemma therefore yields
\begin{align}
    \returnloss(\theta-\alpha g_\eta)-\returnloss(\theta)
    \leq{}&
    -\alpha n_R^2(1+w_\eta c_+)
    \nonumber\\
    &+
    \frac{\beta_R\alpha^2n_R^2}{2}
    \left[1+2w_\eta c_+
    +w_\eta^2(1-c_-^2)\right].
    \label{eq:return_difference_proof}
\end{align}
The right-hand side is strictly negative when
\begin{equation}
    0<\alpha<
    \frac{2(1+w_\eta c_+)}
    {\beta_R[1+2w_\eta c_++w_\eta^2(1-c_-^2)]}.
    \label{eq:exact_descent_bound}
\end{equation}
Because the numerical guard is inactive, \(q\in(0,1]\),
and hence \(w_\eta=q^{1-\eta}\in(0,1]\).
Consequently,
\begin{align}
    1+2w_\eta c_++w_\eta^2(1-c_-^2)
    &\leq 2+2w_\eta c_+
    \nonumber\\
    &=2(1+w_\eta c_+).
\end{align}
The upper bound in \eqref{eq:exact_descent_bound} is therefore
at least \(1/\beta_R\), proving the theorem.
\end{proof}

\subsection{Proof of Theorem~\ref{thm:weak_gradient_advantage}}
\label{app:proof_weak_gradient}

\paragraph{Matched-progress interpretation.}
For \(g_R^\top d>0\), define
\[
    \widehat d
    :=
    \frac{\lVert g_R\rVert_2^2}{g_R^\top d}\,d.
\]
Then \(g_R^\top\widehat d=\lVert g_R\rVert_2^2\), so
\(\widehat d\) and \(g_R\) yield the same first-order
decrease in the native actor loss.
As \(\alpha\to0\), a Taylor expansion gives
\begin{equation}
\begin{aligned}
    &L_T(\theta-\alpha g_R)
    -L_T(\theta-\alpha\widehat d)\\
    &\qquad
    =\alpha g_T^\top(\widehat d-g_R)+o(\alpha)\\
    &\qquad
    =\alpha\lVert g_R\rVert_2^2
    \left[
        \frac{g_T^\top d}{g_R^\top d}
        -
        \frac{g_T^\top g_R}{\lVert g_R\rVert_2^2}
    \right]
    +o(\alpha)\\
    &\qquad
    =\alpha\lVert g_R\rVert_2^2\mathcal G(d)+o(\alpha).
\end{aligned}
\label{eq:gain_interpretation}
\end{equation}
Hence, \(\mathcal G(d)\) measures the additional first-order
temporal decrease per unit primary decrease relative to
the native-gradient update.

\begin{proof}
Write \(w:=q^{1-\eta}\) and \(c_+:=\max\{c,0\}\).
Since \(0<q<1\), the cap is inactive,
\(q=\lVert g_T\rVert_2/\lVert g_R\rVert_2\), and \(0<w\leq1\).
The projection rule gives
\[
    g_\eta
    =
    \bigl(1-w\min\{c,0\}\bigr)g_R
    +q^{-\eta}g_T.
\]
Consequently,
\[
    g_R^\top g_\eta
    =
    \lVert g_R\rVert_2^2(1+wc_+)>0.
\]
Substitution into \eqref{eq:matched_temporal_gain} yields
\begin{equation}
    \mathcal G(g_\eta)
    =
    \frac{q^{2-\eta}(1-c^2)}
         {1+wc_+}.
    \label{eq:damper_exact_gain}
\end{equation}

For scalarization, \(g_\lambda=g_R+\lambda g_T\), and
\[
    g_R^\top g_\lambda
    =
    \lVert g_R\rVert_2^2(1+\lambda qc).
\]
Since \(\lambda q<1-\delta\) and \(\lvert c\rvert\leq1\),
\begin{equation}
    \delta<1+\lambda qc<2-\delta.
    \label{eq:scalar_denominator_bound}
\end{equation}
Thus, the scalarized direction also has positive primary progress.
Direct substitution gives
\begin{equation}
\begin{aligned}
    \mathcal G(g_\lambda)
    &=
    q\frac{c+\lambda q}{1+\lambda qc}-qc\\
    &=
    \frac{\lambda q^2(1-c^2)}
         {1+\lambda qc}.
\end{aligned}
\label{eq:scalar_exact_gain}
\end{equation}

Since \(1\leq1+wc_+\leq2\), \eqref{eq:damper_exact_gain} gives
\begin{equation}
    \frac12(1-c^2)q^{2-\eta}
    \leq
    \mathcal G(g_\eta)
    \leq
    (1-c^2)q^{2-\eta}.
    \label{eq:damper_gain_bounds}
\end{equation}
Similarly, combining \eqref{eq:scalar_exact_gain} with
\eqref{eq:scalar_denominator_bound} yields
\begin{equation}
    \frac{\lambda}{2-\delta}(1-c^2)q^2
    \leq
    \mathcal G(g_\lambda)
    \leq
    \frac{\lambda}{\delta}(1-c^2)q^2.
    \label{eq:scalar_gain_bounds}
\end{equation}
Because \(\lambda>0\) and \(\delta\in(0,1)\) are fixed,
these bounds establish \eqref{eq:weak_gradient_rates},
with positive constants independent of \(q\) and \(c\).
They also cover \(\lvert c\rvert=1\), where both gains vanish.

Finally, if \(\lvert c\rvert<1\), both gains are positive,
and their common factor \(1-c^2\) cancels:
\begin{equation}
    \frac{\mathcal G(g_\eta)}{\mathcal G(g_\lambda)}
    =
    \frac{q^{-\eta}}{\lambda}
    \frac{1+\lambda qc}{1+wc_+}.
    \label{eq:relative_exact_gain}
\end{equation}
Combining \(1\leq1+wc_+\leq2\) with
\eqref{eq:scalar_denominator_bound} gives
\begin{equation}
    \frac{\delta}{2\lambda}q^{-\eta}
    \leq
    \frac{\mathcal G(g_\eta)}{\mathcal G(g_\lambda)}
    \leq
    \frac{2-\delta}{\lambda}q^{-\eta}.
    \label{eq:relative_gain_bounds}
\end{equation}
Thus,
\[
    \frac{\mathcal G(g_\eta)}{\mathcal G(g_\lambda)}
    =\Theta(q^{-\eta}),
\]
with constants independent of \(q\) and \(c\).
For every fixed \(\eta>0\) and \(\lambda>0\), the lower
bound diverges as \(q\to0\), establishing the claimed
asymptotic advantage for noncollinear gradients.
\end{proof}

\section{Additional Fixed-Weight Scalarization Results}
\label{app:scalarization_sweeps}

We report the five TD3 task sweeps omitted from
Section~\ref{sec:scalarization_sweep}: LunarLander, Pendulum, Reacher,
Ant, and Walker. Each sweep uses the 20 fixed weights
\(\Lambda=\{0.05,0.10,\ldots,1.0\}\), with five training seeds and
ten evaluation episodes per seed. Points show mean return and
FFT-based oscillation score; dashed lines show the reported TD3 means
of \method{}. The highlighted point has the lowest mean oscillation
score within each sweep. Axis ranges vary by task.

\begingroup
\newcommand{\scalarizationappendixplot}[4]{%
    \par\noindent
    \begin{minipage}{\linewidth}
        \centering
        \includegraphics[width=0.88\linewidth]{figures/lambda_scalarization_sweep_#1.pdf}\par
        \vspace{2pt}
        \refstepcounter{figure}\label{#2}%
        {\small\raggedright\noindent
        \textbf{Figure \thefigure: Fixed-weight scalarization on TD3 #3.}
        #4\par}
    \end{minipage}\par
}

\smallskip
\scalarizationappendixplot{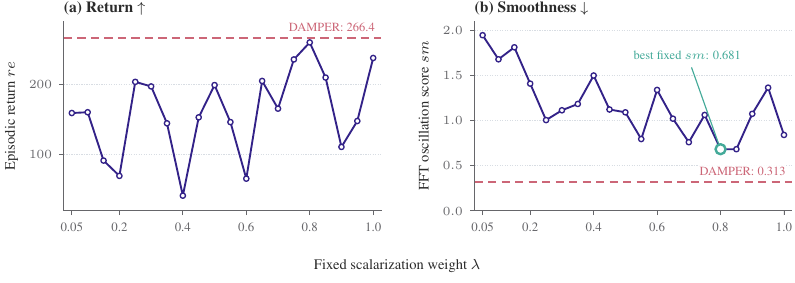}{fig:lambda_scalarization_lunarlander}{LunarLander}{%
    The lowest scalarization score is \(sm=0.681\) at
    \(\lambda=0.80\), compared with \(0.313\) for \method{}.}

\medskip
\scalarizationappendixplot{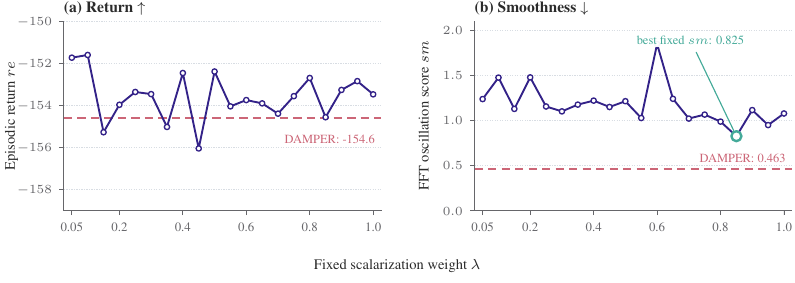}{fig:lambda_scalarization_pendulum}{Pendulum}{%
    The lowest scalarization score is \(sm=0.825\) at
    \(\lambda=0.85\), compared with \(0.463\) for \method{}.}

\medskip
\scalarizationappendixplot{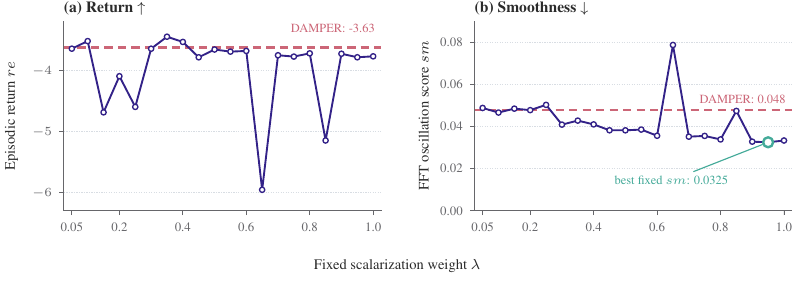}{fig:lambda_scalarization_reacher}{Reacher}{%
    Reacher is the exception: scalarization reaches \(sm=0.0325\)
    at \(\lambda=0.95\), below the reported \(0.048\) for \method{}.}

\medskip

\scalarizationappendixplot{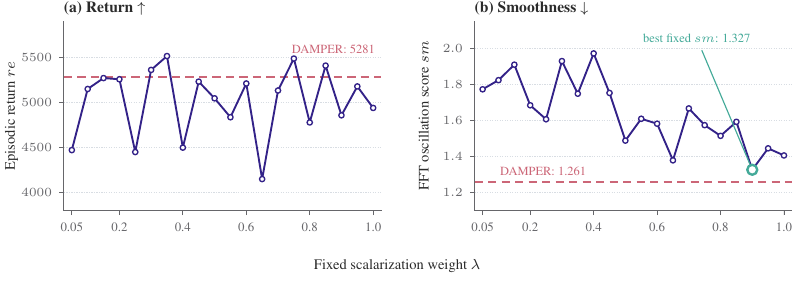}{fig:lambda_scalarization_ant}{Ant}{%
    The lowest scalarization score is \(sm=1.327\) at
    \(\lambda=0.90\), compared with \(1.261\) for \method{}.}

\medskip
\scalarizationappendixplot{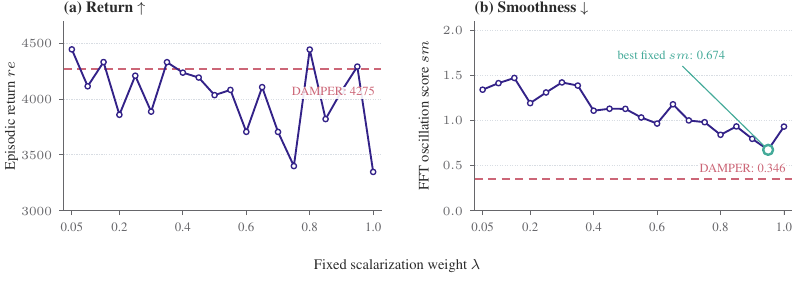}{fig:lambda_scalarization_walker}{Walker}{%
    The lowest scalarization score is \(sm=0.674\) at
    \(\lambda=0.95\), compared with \(0.346\) for \method{}.}

\paragraph{Comparison with smoothing baselines.}
We compare tuned scalarization with the five smoothing baselines in
Table~\ref{tab:td3-results}, excluding \method{}. On Ant, the weight \(\lambda=0.75\)
gives \((re,sm)=(5486,1.574)\). Its return exceeds the highest
baseline return (\(5370\), CAPS), while its oscillation score is
\(8.1\%\) below the lowest baseline score (\(1.713\), Grad-CAPS).
Thus, it improves both metrics over all five baselines.
On LunarLander, \(\lambda=0.80\) gives \((259.8,0.681)\),
improving both metrics over Grad-CAPS, ASAP, L2C2, and PAVE;
PAVE, for example, obtains \((238.1,0.841)\).
On Hopper, \(\lambda=0.80\) gives \((3338,0.714)\), improving
both metrics over CAPS, ASAP, L2C2, and PAVE. Its oscillation score
is \(30.1\%\) below the best baseline score (\(1.022\), PAVE),
with a similar return (\(3338\) versus \(3328\)).
On Pendulum, \(\lambda=0.85\) gives \((-154.6,0.825)\),
improving on both ASAP (\(-155.3,2.013\)) and L2C2
(\(-156.0,3.293\)) in both metrics.

These gains are task dependent. CAPS remains smoother on LunarLander,
and CAPS and PAVE remain smoother on Pendulum. The lowest scalarization
scores on Reacher (\(0.0325\)) and Walker (\(0.674\)) improve on
all five baselines, but their returns are lower: \(-3.79\) versus
\(-3.55\) to \(-3.34\) on Reacher, and \(4293\) versus
\(4880\) to \(5046\) on Walker. These results make a tuned
temporal-consistency loss a strong
baseline: additional smoothing mechanisms do not consistently improve
the return--smoothness trade-off in these experiments. Comparisons use
reported means and weights selected from the sweep; they do not
establish statistical significance.
\par
\endgroup
\FloatBarrier
\section{Experimental Details}
\label{app:experimental_hyperparameters}

\paragraph{Baseline configurations.}
We follow the baseline hyperparameter settings reported in the PAVE
appendix~\parencite{lee2026pave}. Table~\ref{tab:baseline_hyperparameters}
lists the settings for CAPS, Grad-CAPS, ASAP, and PAVE from
Appendix E.1 (Tables 10 and 11) of that work. Grad-CAPS is denoted
as GRAD in the original tables. CAPS, Grad-CAPS, and ASAP use the same
settings under TD3 and SAC; PAVE uses backbone-specific loss weights.

\begin{table}[!htbp]
    \centering
    \caption{Baseline hyperparameters adopted from PAVE.
    Settings apply to both TD3 and SAC unless a backbone is specified.
    Parameter notation follows each baseline.}
    \label{tab:baseline_hyperparameters}
    \small
    \setlength{\tabcolsep}{3.5pt}
    \renewcommand{\arraystretch}{1.12}
    \begin{tabular}{@{}llcccccc@{}}
        \toprule
        Method & Parameter & LunarLander & Pendulum & Reacher & Ant & Hopper & Walker \\
        \midrule
        CAPS & \(\lambda_T\) & 0.1 & 1.0 & 0.1 & 0.1 & 0.1 & 0.1 \\
             & \(\lambda_S\) & 0.5 & 5.0 & 0.5 & 0.5 & 0.5 & 0.5 \\
             & \(\sigma\)    & 0.2 & 0.2 & 0.2 & 0.2 & 0.2 & 0.2 \\
        \addlinespace[3pt]
        Grad-CAPS & \(\lambda_T\) & 1.0 & 1.0 & 1.0 & 1.0 & 1.0 & 1.0 \\
        \addlinespace[3pt]
        ASAP & \(\lambda_T\) & 0.005 & 0.005 & 0.1 & 0.05 & 0.07 & 0.05 \\
             & \(\lambda_S\) & 0.03  & 0.03  & 0.1 & 0.3  & 0.3  & 0.3 \\
             & \(\lambda_P\) & 2.0   & 2.0   & 2.0 & 2.0  & 2.0  & 2.0 \\
        \midrule
        PAVE (TD3) & \(\lambda_1\) & 0.1  & 2.0   & 0.1  & 0.1   & 0.1   & 0.1 \\
                   & \(\lambda_2\) & 0.1  & 0.005 & 0.1  & 0.005 & 0.005 & 0.1 \\
                   & \(\lambda_3\) & 0.01 & 2.0   & 0.01 & 0.5   & 0.5   & 0.01 \\
        \addlinespace[3pt]
        PAVE (SAC) & \(\lambda_1\) & 0.1  & 0.1   & 0.1    & 0.1    & 2.0    & 2.0 \\
                   & \(\lambda_2\) & 0.5  & 0.005 & 0.0005 & 0.0005 & 0.0005 & 0.005 \\
                   & \(\lambda_3\) & 0.05 & 0.5   & 1.0    & 1.0    & 3.0    & 2.0 \\
        \bottomrule
    \end{tabular}
    \par\vspace{4pt}
    \raggedright
    For PAVE, \(\lambda_1\), \(\lambda_2\), and \(\lambda_3\) weight
    mixed-partial regularization, vector-field consistency, and curvature
    preservation, respectively. Its perturbation scale \(\sigma=0.01\)
    and curvature floor \(\delta=1.0\) are fixed across all tasks
    and both backbones.
\end{table}
\FloatBarrier

\paragraph{L2C2.}
For both TD3 and SAC, we use \(\sigma=1.0\),
\(\underline{\lambda}=0.01\), \(\overline{\lambda}=1.0\), and
\(\beta=0.1\) on all six tasks.

\paragraph{\method{}.}
Table~\ref{tab:damper_eta_hyperparameters} lists the interpolation
parameter \(\eta\) used for each task--backbone pair in our main
experiments. The parameter sweeps used for ablation are reported separately.

\begin{table}[!htbp]
    \centering
    \caption{Interpolation parameter \(\eta\) for \method{}.}
    \label{tab:damper_eta_hyperparameters}
    \small
    \setlength{\tabcolsep}{7pt}
    \renewcommand{\arraystretch}{1.15}
    \begin{tabular}{@{}lcccccc@{}}
        \toprule
        Backbone & LunarLander & Pendulum & Reacher & Ant & Hopper & Walker \\
        \midrule
        TD3 & 0.8 & 1.0 & 1.0 & 0.0 & 1.0 & 1.0 \\
        SAC & 1.0 & 1.0 & 0.8 & 0.0 & 0.0 & 1.0 \\
        \bottomrule
    \end{tabular}
\end{table}
\FloatBarrier

\paragraph{Hardware.}
Experiments were conducted using NVIDIA GeForce RTX 2080 and
RTX 5060 Ti GPUs.
\begin{table}[!htbp]
    \centering
    \caption{Training budgets and network sizes following the PAVE
    protocol~\parencite{lee2026pave}.}
    \label{tab:training_configuration}
    \small
    \setlength{\tabcolsep}{12pt}
    \renewcommand{\arraystretch}{1.1}
    \begin{tabular}{lcc}
        \toprule
        Setting & TD3 & SAC \\
        \midrule
        \multicolumn{3}{l}{\textit{Training iterations}} \\
        LunarLander & 500,000   & 500,000   \\
        Pendulum    & 100,000   & 100,000   \\
        Reacher     & 500,000   & 500,000   \\
        Ant         & 1,000,000 & 1,000,000 \\
        Hopper      & 1,000,000 & 1,000,000 \\
        Walker2d    & 1,000,000 & 1,000,000 \\
        \midrule
        Hidden-layer widths & \((400,300)\) & \((256,256)\) \\
        \bottomrule
    \end{tabular}
\end{table}

\section{Native Actor Objectives}
\label{app:native_actor_objectives}

We specify the native actor losses and deterministic action maps used
in our TD3 and SAC implementations. Expectations over states use the
state marginal of the replay transition distribution \(\mathcal D\).
The critic parameters are held fixed during actor optimization, while
gradients pass through the critic's action input.

\paragraph{TD3.}
For TD3~\parencite{fujimoto2018addressing}, the deterministic actor
\(\mu_\theta\) is optimized using
\begin{equation}
    \returnloss^{\mathrm{TD3}}(\theta)
    =-\mathbb E_{s\sim\mathcal D}
    \left[Q_{\phi_1}(s,\mu_\theta(s))\right].
    \label{eq:td3_actor_loss}
\end{equation}
The temporal-consistency loss uses
\(\bar\mu_\theta(s)=\mu_\theta(s)\).

\paragraph{SAC.}
For SAC~\parencite{haarnoja2018soft}, let \(a_\theta(s,\epsilon)\)
denote a reparameterized action with
\(\epsilon\sim\mathcal N(0,I)\), and let \(\alpha_{\mathrm{ent}}\)
denote the entropy temperature. The native actor loss is
\begin{equation}
\begin{split}
    \returnloss^{\mathrm{SAC}}(\theta)
    =\mathbb E_{s\sim\mathcal D,\epsilon}
    \big[&\alpha_{\mathrm{ent}}
        \log\pi_\theta(a_\theta(s,\epsilon)\mid s)\\
        &-\min_{j\in\{1,2\}}
        Q_{\phi_j}(s,a_\theta(s,\epsilon))\big].
    \label{eq:sac_actor_loss}
\end{split}
\end{equation}
The temporal-consistency loss uses the squashed mean action
\(\pi_\theta(s)=a_\theta(s,0)\), including any action-bound
rescaling used by the policy.

\Needspace{28\baselineskip}
\section{Computational Cost}
\label{app:computational_cost}

We measure the per-iteration update time of TD3 and SAC and their
smoothness variants on an NVIDIA GeForce RTX 5060 Ti GPU using
Hopper-v5 transitions collected under random actions.
Each configuration starts from freshly initialized networks
 and uses a batch size of 256. We perform 50 warmup
iterations followed by 1,000 timed iterations. The measured operations
include forward computation, loss evaluation, gradient computation,
and optimizer steps; logging, diagnostics, and target-network copy
operations are excluded. During the timed interval, TD3 performs
500 actor updates and SAC performs 1,000. We therefore report time
per update iteration and normalize it to the unregularized baseline
within each backbone.

\begin{table}[!htbp]
    \centering
    \caption{Average update time on Hopper-v5. Relative time is the ratio
    to the corresponding unregularized backbone. Lower is better.
    Each entry averages 1,000 timed update iterations.}
    \label{tab:runtime_comparison}
    \small
    \begin{tabular}{lrrrr}
        \toprule
        & \multicolumn{2}{c}{TD3}
        & \multicolumn{2}{c}{SAC} \\
        \cmidrule(lr){2-3}\cmidrule(lr){4-5}
        Method & Time (ms) & Relative ($\times$)
               & Time (ms) & Relative ($\times$) \\
        \midrule
        Unregularized       & 4.54  & 1.00 & 10.87 & 1.00 \\
        CAPS                & 7.32  & 1.61 & 18.78 & 1.73 \\
        Grad-CAPS           & 6.35  & 1.40 & 18.60 & 1.71 \\
        ASAP                & 9.85  & 2.17 & 23.38 & 2.15 \\
        L2C2                & 12.71 & 2.80 & 48.36 & 4.45 \\
        PAVE                & 24.07 & 5.30 & 32.22 & 2.96 \\
        \method{}& 6.57  & 1.45 & 14.26 & 1.31 \\
        \bottomrule
    \end{tabular}
\end{table}

Table~\ref{tab:runtime_comparison} shows that \method{} with
$\eta=0$ requires 6.57\,ms per iteration with TD3 and 14.26\,ms
with SAC, corresponding to overheads of 44.8\% and 31.3\%,
respectively, over the unregularized backbones.
On TD3, \method{} is slightly slower than Grad-CAPS but faster
than CAPS, ASAP, L2C2, and PAVE. On SAC, it has the lowest
measured runtime among the compared smoothness methods.
These measurements characterize update computation under the
tested configuration and do not represent end-to-end training time.

\end{document}